\documentclass[letterpaper, 10 pt, conference]{ieeeconf}

\IEEEoverridecommandlockouts

\usepackage[T1]{fontenc}
\usepackage{cite}
\usepackage{gensymb}
\usepackage{amsmath,amssymb,amsfonts}
\usepackage{algorithmic}
\usepackage{graphicx}
\usepackage{textcomp}
\usepackage{xcolor}
\usepackage{siunitx}

\usepackage{hyperref}

\usepackage[normalem]{ulem} 

\usepackage[caption=false, font=footnotesize]{subfig}

\usepackage{custom_commands}

\newtheorem{proposition}{Proposition}

\begin{document}

\title{\LARGE \bf
Kinematic Analysis and Design of a Novel (6+3)-DoF Parallel Robot with Fixed Actuators
}

\author{Arda Yi\u{g}it, David Breton, Zhou Zhou, Thierry Lalibert\'e and Cl\'ement Gosselin %
\thanks{This work was supported by the Natural Sciences and Engineering Research Council of Canada (NSERC). }%
\thanks{The authors are with the Department of Mechanical Engineering, Universit\'e Laval, Qu\'ebec, Qc, Canada. Emails: {\tt\small arda.yigit.1@ulaval.ca}, {\tt\small david.breton.2@ulaval.ca}, {\tt\small zhou.zhou.2@ulaval.ca}, {\tt\small thierry@gmc.ulaval.ca}, {\tt\small gosselin@gmc.ulaval.ca}. }%
}

\maketitle
\thispagestyle{empty}
\pagestyle{empty}

\begin{abstract}
A novel kinematically redundant (6+3)-DoF parallel robot is presented in this paper. Three identical 3-DoF RU/2-RUS legs are attached to a configurable platform through spherical joints. With the selected leg mechanism, the motors are mounted at the base, reducing the reflected inertia. The robot is intended to be actuated with direct-drive motors in order to perform intuitive physical human-robot interaction. The design of the leg mechanism maximizes the workspace in which the end-effector of the leg can have a 2g acceleration in all directions. All singularities of the leg mechanism are identified under a simplifying assumption. A CAD model of the (6+3)-DoF robot is presented in order to illustrate the preliminary design of the robot. 
\end{abstract}

\section{Introduction}

Parallel robotic architectures allow motors to be fixed at the base of the robots. The moving mass can therefore be small, and the robot can be actuated by so-called "direct-drive" motors. These motors do not use a gearbox, thereby yielding a better efficiency, reducing friction and greatly reducing the reflected inertia. They allow for developing backdrivable robots which pave the way for applications such as sensorless physical human-robot interaction. Indeed, lightweight backdrivable parallel mechanisms actuated by direct-drive motors enable the control of the force at the end-effector without being limited by the bandwidth of force sensors.

On the other hand, parallel robots often suffer from mechanical interference and singularities in their reachable workspace, which restricts their orientation capabilities. For example, the tilt angle of a Gough-Stewart platform is typically limited to approximately 45 degrees \cite{Gosselin2018}. This limitation can be circumvented by introducing redundancy, and more precisely actuation or kinematic redundancy \cite{Notash2003}. Actuation redundancy corresponds to using more actuators than the degree of mobility of the mechanism. It is often implemented by adding extra legs to the mechanism \cite{HuiCheng2003} or by actuating passive joints of an existing architecture \cite{Firmani2004}. Actuation redundancy can be useful to avoid singularities \cite{Kurtz1992} or to improve the stiffness of the robot \cite{Kim1997}. Nevertheless, redundant actuation generates internal antagonistic forces in the mechanism, which makes the control of the robot more complex and may require load cells \cite{Harada2013} or force control algorithms \cite{Harada2010}. A parallel robot is said to be kinematically redundant if its degree of mobility is greater than the number of degrees of freedom of the end-effector. Kinematic redundancy is often obtained by introducing additional actuated joints in one or more legs of the parallel mechanism. Therefore, a desired pose of the end-effector corresponds to infinitely many joint configurations. Kinematic redundancy allows for avoiding singularities \cite{Zanganeh1994}, enlarging the workspace (especially in orientation) and using configurable platforms as end-effector platform \cite{Lambert2010}. 

In recent years, several kinematically redundant six-degree-of-freedom (6-DoF) mechanisms have emerged \cite{Gosselin2016,Gosselin2018}. 
Schreiber and Gosselin introduced a (6+3)-DoF parallel robot with six legs, similar to a Gough-Stewart platform \cite{Schreiber2019a}. Three of the six legs of the Gough-Stewart platform are replaced by redundant legs comprising two sub-legs, each containing a prismatic joint. A revolute joint links the two sub-legs and is connected to the moving platform through an extra link and a spherical joint. This new architecture avoids singularities and extends the workspace of the mechanism. 
Wen et al. proposed a backdrivable, three-legged, (6+3)-DoF hybrid parallel robot for sensorless physical human-robot interaction \cite{Wen2021}. Using only three legs reduces interference. The robot is designed to avoid any type II (or parallel) singularity within its workspace, yielding a large translational and orientational workspace. Each leg is actuated by one fixed motor and two mobile motors positioned close to the base and driving a planar five-bar mechanism. In order to enable intuitive physical human-robot interaction with this architecture, the mobile motors must be lightweight, which limits the dynamic capabilities of the robot. 

In this work, we propose a new (6+3)-DoF parallel robot, in which the hybrid parallel leg proposed by Wen et al. \cite{Wen2021} is replaced by a parallel one. As described in Section \ref{sec:robot_architecture}, each leg is actuated by three collinear revolute joints and, therefore, the workspace of the leg has a circular symmetry. The inverse kinematics of a leg are derived and the Jacobian matrices are obtained using screw theory (Section \ref{sec:kinematic_modeling}). The singularities of the leg mechanism are discussed and a simplifying assumption on robot geometry is used to provide a simple condition for type II singularities (Section \ref{sec:singularity_analysis}). The circular symmetry of the leg mechanism is exploited to find the optimal geometric parameters maximizing the workspace in which the end-effector of the leg can have a $2g$ acceleration in any direction (Section \ref{sec:optimal_design}). Finally, a three-dimensional model of the robot obtained with a computer-aided design (CAD) software and the selected hardware is shown in Section \ref{sec:prototyping} in order to present a preliminary design.

\section{Robot Architecture} \label{sec:robot_architecture}

The architecture of the robot is illustrated in Fig. \ref{fig:robot_architecture}. The robot includes three identical 3-DoF \underline{R}U/2-\underline{R}US legs (one \underline{R}U chain and two \underline{R}US chains), shown in Fig. \ref{fig:leg_parametrization}. Here, R represents a revolute joint, U a universal joint and S a spherical joint. The actuated joints are underlined. The axes of the actuated R joints in each leg are collinear, resulting in a large workspace with a circular symmetry. All motors are mounted at the base, which allows for a low inertia of the mobile parts and consequently favours intuitive and safe physical human-robot interaction. 

\begin{figure}[htbp]
    \centering
    \includegraphics[width=0.45\textwidth]{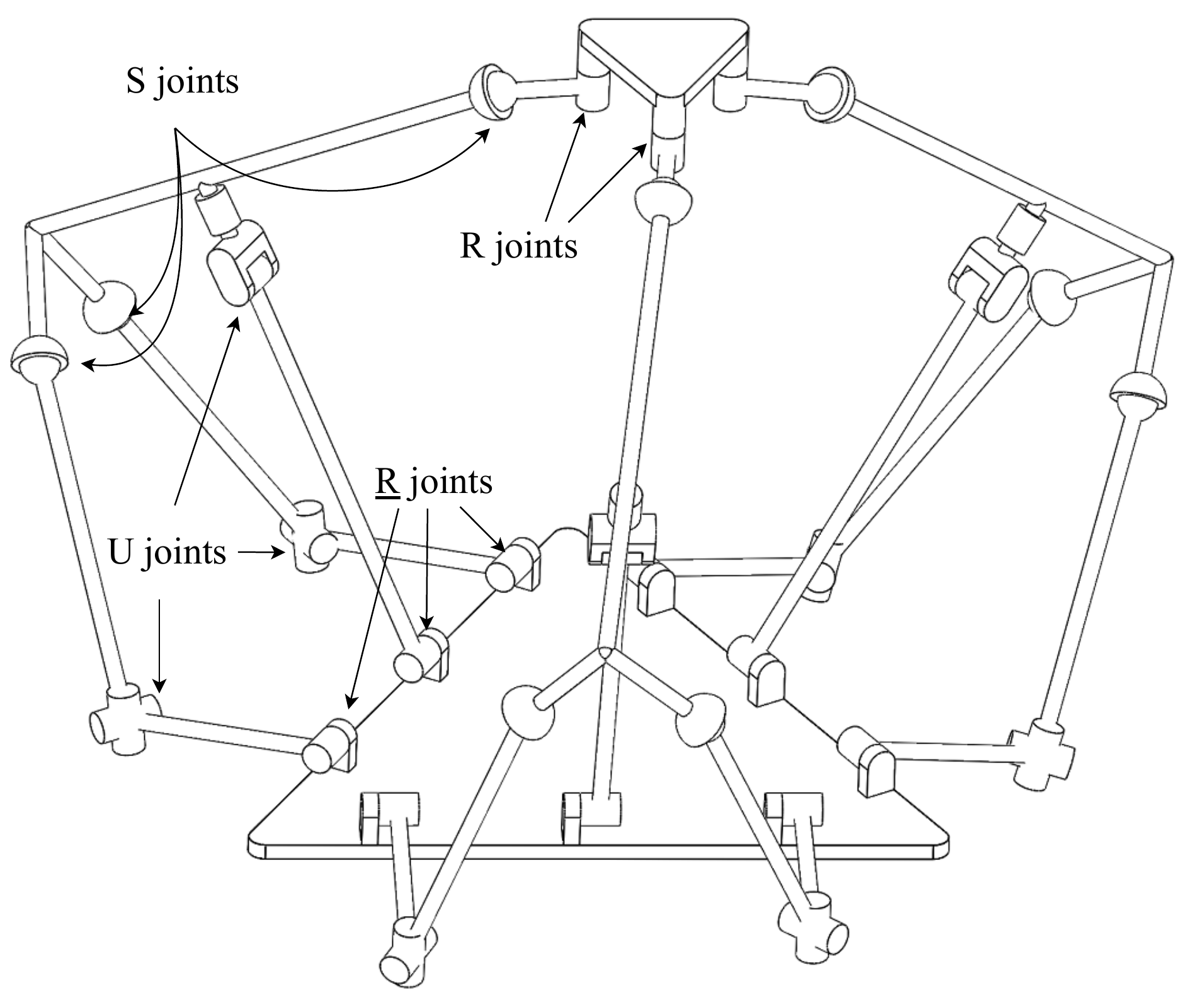}
    \caption{Parallel robot architecture. }
    \label{fig:robot_architecture}
\end{figure}

\begin{figure}[htbp]
    \centering
    \includegraphics[width=0.45\textwidth]{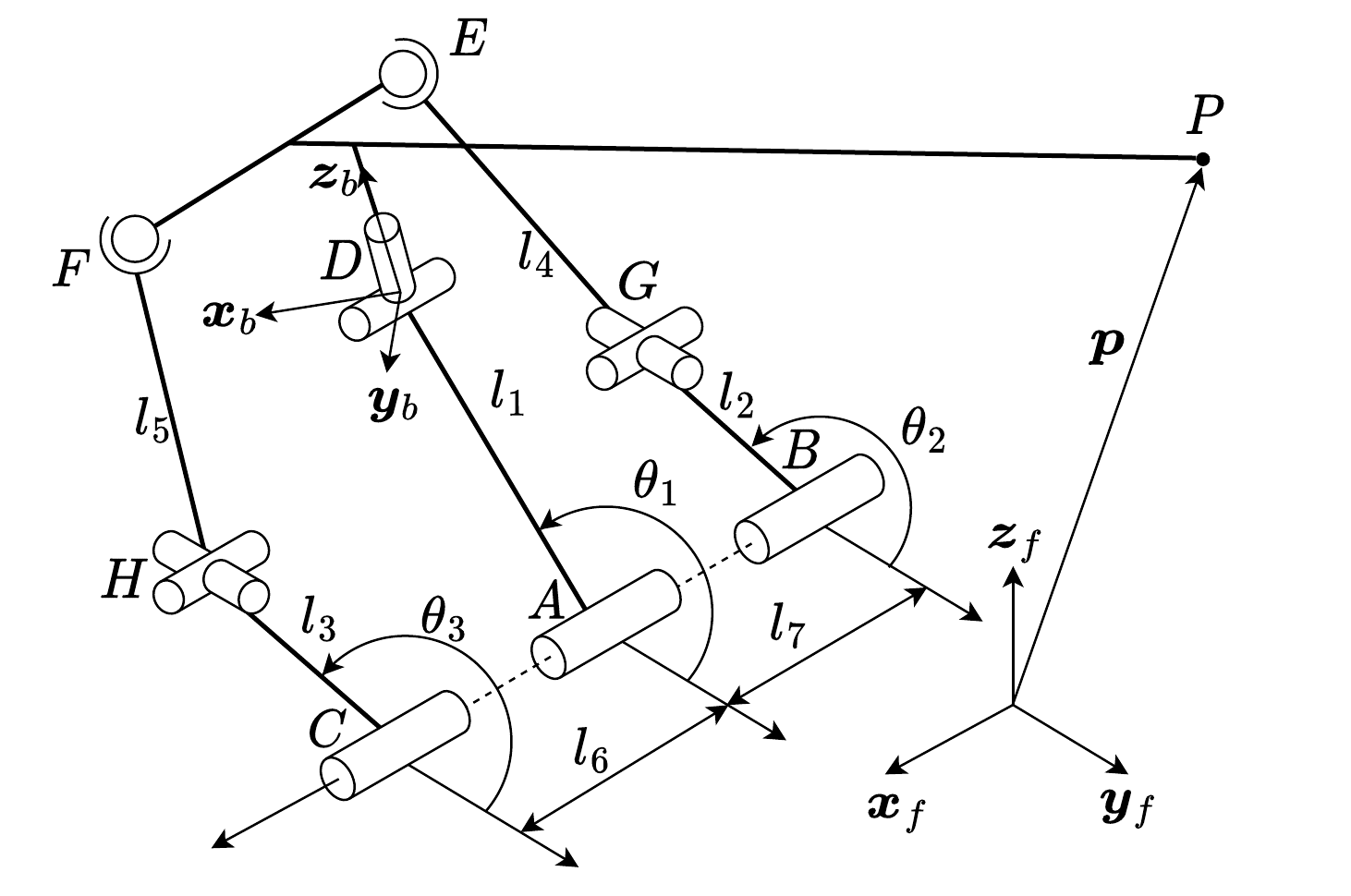}
    \caption{Parameterization of the leg mechanism. }
    \label{fig:leg_parametrization}
\end{figure}

The \underline{R}U chain constrains the motion of point $P$ to a sphere of radius $d_p$ around point $D$, where $d_p$ stands for the constant distance between points $D$ and $P$. 
Combining the circular symmetry of the leg and the constraint on point $P$, and assuming that the geometric parameters of the \underline{R}US chains are chosen appropriately, the workspace of the leg corresponds to a torus. 
The tube of the torus has a radius $d_p$. The distance from the centre of the tube to the centre of the torus is $l_1$, the length of the proximal link of the $\underline{R}U$ chain. 
It should be noted that, if $l_1 < d_p$, then the workspace corresponds to a spindle torus, which is not simply connected (i.e. has a "hole" inside). Indeed, point $A$ and its neighbourhood are not reachable. 

The constant orientation workspace (COW) is the set of positions that can be reached by the end-effector while keeping a constant orientation \cite{Merlet1998}. The COW of the kinematic chain of the robot (not restricted by constraints that depend on the physical implementation) is obtained by taking the intersection of the workspace of each leg, with an offset depending on the mobile platform. An example for the case $d_p < l_1$, generated with a CAD software, is illustrated in Fig. \ref{fig:ws_cad}. Figure \ref{fig:ws_cad_legs} shows the kinematic workspace of each leg and their intersection is shown in Fig. \ref{fig:ws_cad_intersection}. Only the upper half is shown since the lower half cannot be used in practice due to interference with the base. 

\begin{figure}[htbp]
    \centering
    \subfloat[\label{fig:ws_cad_legs}Kinematic workspace of the legs of the $(6+3)$-DoF robot.]{%
        \includegraphics[trim=400px 100px 350px 50px,clip=true,width=0.49\linewidth]{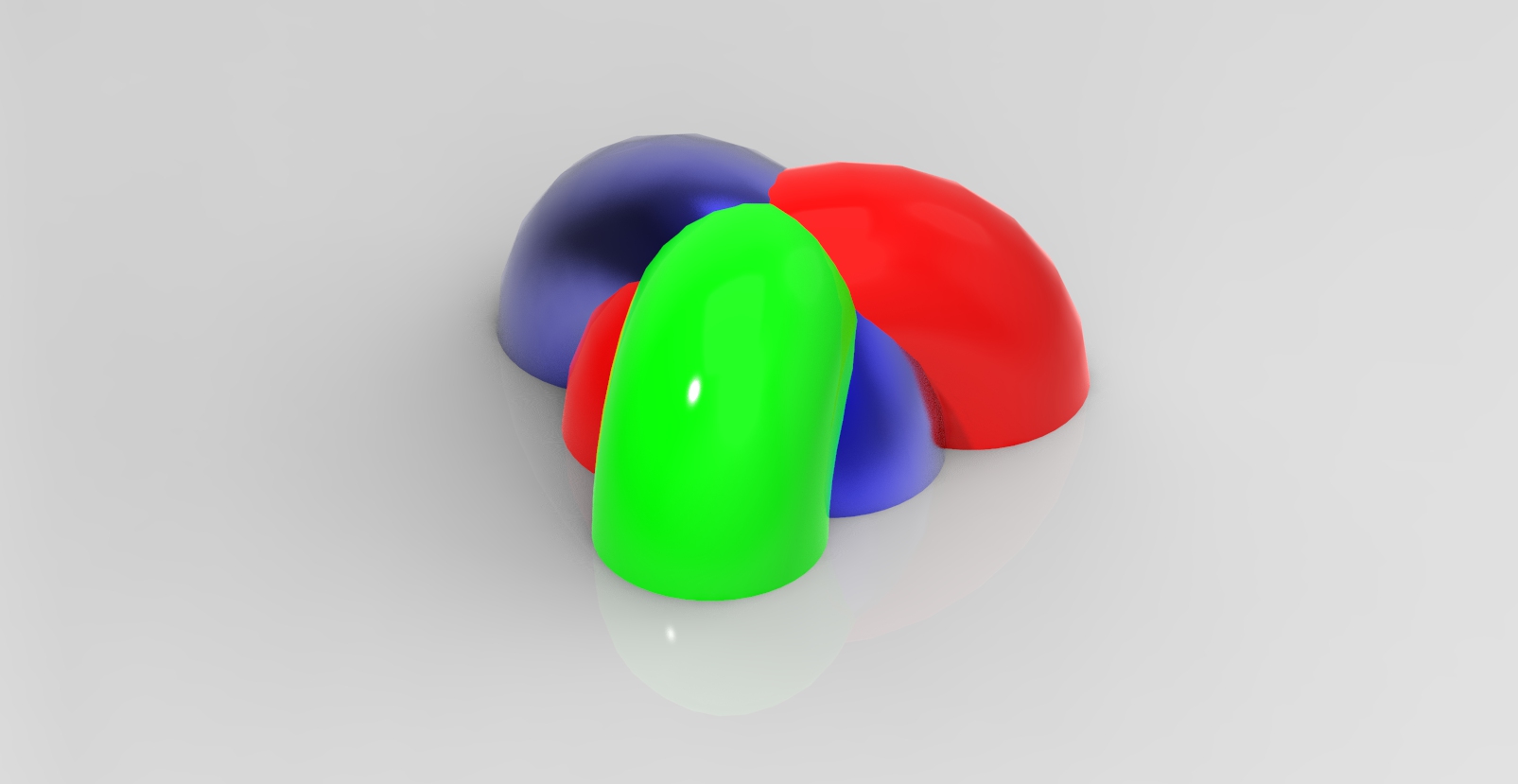}}
    \hfill
    \subfloat[\label{fig:ws_cad_intersection}Intersection of the kinematic workspaces of the legs.]{%
        \includegraphics[trim=400px 100px 350px 50px,clip=true,width=0.49\linewidth]{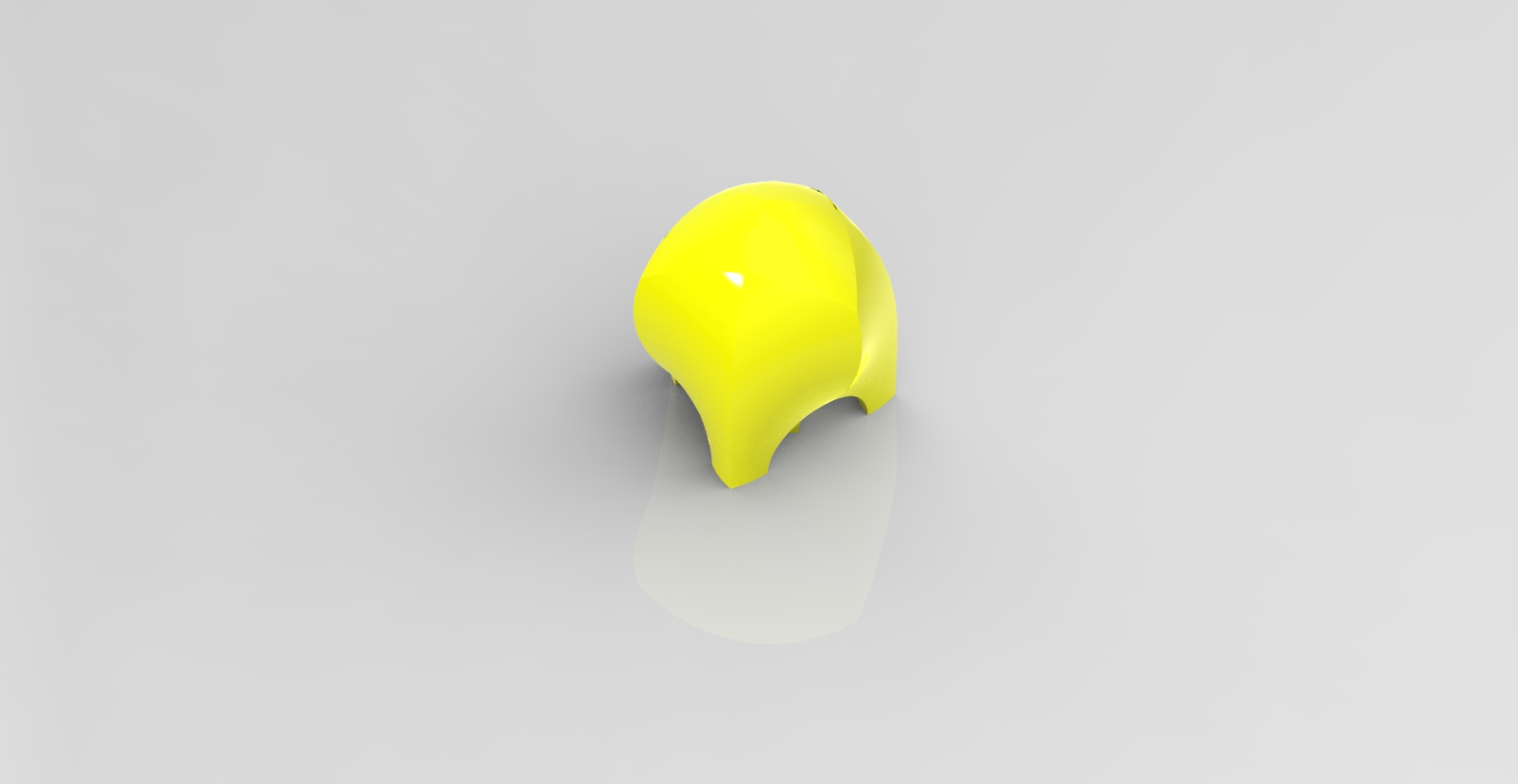}}
    \caption{Constant orientation workspace of the $(6+3)$-DoF robot determined as the intersection of the kinematic workspaces of the legs.}
    \label{fig:ws_cad}
\end{figure}

The mobile platform is the same as the one proposed by Wen et al. \cite{Wen2021}. It is equipped with three revolute joints having parallel axes and driving links that are in turn connected to the spherical joints at the end-effector of each leg. Therefore, the robot has nine DoFs and nine actuators. There is no actuation redundancy that may lead to antagonistic internal forces. It should also be noted that other configurable platform architectures using prismatic joints are also possible \cite{Wen2019,Nguyen2021}.

\section{Kinematic Modelling of the Leg Mechanism} \label{sec:kinematic_modeling}

\subsection{Parameterization}

The geometric parameters of the robot are shown in Fig. \ref{fig:leg_parametrization}. The index referring to the leg number is omitted (except otherwise stated) to simplify the notations. 

The base of the robot is fixed with respect to an inertial reference frame $(O,\vv{x}_f,\vv{y}_f,\vv{z}_f)$ whose $\vv{x}_f$ axis is parallel to the axes of the actuated joints. For any point represented with an uppercase letter, its position vector in the base frame is written in bold lowercase (e.g., the position vector of point $P$ is noted $\vv{p}$). 
A moving body frame $(D,\vv{x}_b,\vv{y}_b,\vv{z}_b)$ is attached to the end-effector platform. Its $\vv{x}_b$ axis is aligned with the line $\vv{p}-\vv{d}$ and its $\vv{z}_b$ axis aligned with one of the axes of the U joint at $D$ as shown in Fig. \ref{fig:leg_parametrization}. The coordinates of $P$, $E$ and $F$ in the body frame are respectively written as $\vv{p}_b=[ d_p \; 0 \; 0 ]^T$, $\vv{e}_b=[ e_x \; e_y \; e_z ]^T$ and $\vv{f}_b=[ f_x \; f_y \; f_z ]^T$. The rotation matrix $\vv{Q}$ describes the orientation of the body frame with respect to the base frame. Parameters denoted as $l_j$ correspond to the length of the links. Actuated joint angles are $\vv{\theta} = [ \theta_1 \; \theta_2 \; \theta_3 ]^T$. Point $P$ is considered the end-effector of the leg and its position vector is written as $\vv{p} = [ x \; y \; z ]^T$. For an angle $\theta_j$, its cosine and sine are written $c_j=\cos \theta_j$ and $s_j=\sin \theta_j$. Unless otherwise stated, all vectors are expressed in the base frame.

\subsection{Inverse Kinematics}

Zhou and Gosselin proposed a solution of the inverse kinematics problem for a similar leg mechanism in a previous work \cite{Zhou2022}. The solution is adapted here to the selected geometry, with collinear actuators. 
For reasons of brevity, only the main steps of the resolution are discussed in the following. Only the interior of the reachable workspace for the solution of the inverse kinematics is considered here to avoid situations in which one of the equations may degenerate, namely type I singularities \cite{Gosselin1990}. 

\subsubsection{RU Chain}

Let $\vv{R}_{\vv{u},\psi}$ be the rotation matrix of axis $\vv{u}$ and angle $\psi$ and $(\vv{x},\vv{y},\vv{z})$ the canonical basis of $\mathbb{R}^3$. The rotation matrix $\vv{Q}$ can be obtained using the joint variables $\vv{\theta}_s = [ \theta_1 \; \theta_{s2} \; \theta_{s3} ]^T$ of the RU chain, where $\theta_{s2}$ and $\theta_{s3}$ are the joint angles associated with the two R joints of the U joint. One has
\begin{equation}
    \vv{Q} = 
    \vv{R}_{(\vv{x},\theta_1)}
    \vv{R}_{(\vv{x},\theta_{s2})}
    \vv{R}_{(\vv{z},\theta_{s3}-\frac{\pi}{2})}
\end{equation}
or, more explicitly: 
\begin{equation}
    \vv{Q} = 
    \begin{bmatrix}
        s_{s3} & c_{s3} & 0 \\
        -c_{1+s2}c_{s3} & c_{1+s2}s_{s3} & -s_{1+s2} \\
        -s_{1+s2}c_{s3} & s_{1+s2}s_{s3} & c_{1+s2}
    \end{bmatrix}
\end{equation}
where $c_{1+s2}=\cos(\theta_1+\theta_{s2})$ and $s_{1+s2}=\sin(\theta_1+\theta_{s2})$. 

Expanding the constraint
\begin{equation} \label{eq:constraint_RU_pd}
    (\vv{p}-\vv{d})^T(\vv{p}-\vv{d}) = {\vv{p}_b}^T\vv{p}_b
\end{equation}
yields a trigonometric equation featuring $\cos \theta_1$ and $\sin \theta_1$. Two possible values for $\theta_1$ are the solutions of the polynomial equation obtained using the tangent half-angle formulae. 

The joint angles $\theta_{s2}$ and $\theta_{s3}$ of the RU chain are needed in order to obtain the expression of the rotation matrix $\vv{Q}$ from the end-effector position $\vv{p}$. 

The relation
\begin{equation}
    \vv{Q}\vv{p}_b = \vv{p} - \vv{d}
\end{equation}
yields a system of trigonometric equations featuring $\cos \theta_{s3}$, $\sin \theta_{s3}$, $\cos(\theta_1+\theta_{s2})$ and $\sin(\theta_1+\theta_{s2})$. Since $\theta_1$ is known, the value of $\theta_{s2}$ is obtained using the $\atantwo$ function and the two possible values of $\theta_{s3}$ are obtained using the tangent half-angle formulae as previously. 

\subsubsection{RUS Chains}

In the same vein as \eqref{eq:constraint_RU_pd}, expanding the constraint equations \begin{equation}
    \begin{cases}
        (\vv{e}-\vv{g})^T(\vv{e}-\vv{g}) = {l_4}^2 \\
        (\vv{f}-\vv{h})^T(\vv{f}-\vv{h}) = {l_5}^2
    \end{cases}
\end{equation}
yields two trigonometric equations. Solving them with the help of the tangent half-angle formulae results in two possible values for both $\theta_2$ and $\theta_3$.

\subsection{Jacobian Matrices}

The Jacobian matrices are obtained using screw theory. 

Let $\vv{\xi}_{EE}$ be the end-effector twist and $\vv{\xi}_J$ the twist associated with the joint located at point $J$. One can then write
\begin{equation} \label{eq:twist_ee}
    \begin{cases}
        \vv{\xi}_{EE} = \vv{\xi}_A + \vv{\xi}_D \\
        \vv{\xi}_{EE} = \vv{\xi}_B + \vv{\xi}_G + \vv{\xi}_E \\
        \vv{\xi}_{EE} = \vv{\xi}_C + \vv{\xi}_H + \vv{\xi}_F
    \end{cases}
\end{equation}

Let us define the reciprocal product $\cdot$ of twists $\vv{\xi}_1$ and $\vv{\xi}_2$ by: 
\begin{equation}
    \vv{\xi}_1 \cdot \vv{\xi}_2 = \vv{\omega}_1^T\vv{v}_{O,2} + \vv{\omega}_2^T\vv{v}_{O,1}
\end{equation}
where $\vv{\xi}_1=\begin{bmatrix} \vv{\omega}_1 \\ \vv{v}_{O,1} \end{bmatrix}_O$ and $\vv{\xi}_2=\begin{bmatrix} \vv{\omega}_2 \\ \vv{v}_{O,2} \end{bmatrix}_O$ are expressed with respect to any common reference point $O$. Two twists are reciprocal if $\vv{\xi}_1 \cdot \vv{\xi}_2 = 0$. 

Any zero-pitch twist whose line (or instantaneous screw axis) passes through the axis of an R joint is reciprocal to the twist of that joint. Hence, for a U or S joint, the axis of the reciprocal twist passes through the centre of the joint. 

Let $\vv{\xi}_{JK}$ be a zero-pitch twist of line $\vv{k}-\vv{j}$ and amplitude $\| \vv{k}-\vv{j} \|$. This twist is reciprocal to the twists of R, U or S joints (if any) at points $J$ and $K$. 
Hence, $\vv{\xi}_{GE}$ is reciprocal to $\vv{\xi}_{G}$ and $\vv{\xi}_{E}$, $\vv{\xi}_{HF}$ is reciprocal to $\vv{\xi}_{H}$ and $\vv{\xi}_{F}$, and $\vv{\xi}_{DP}$ is reciprocal to $\vv{\xi}_{D}$. 
Therefore, from \eqref{eq:twist_ee}: 
\begin{equation} \label{eq:twist_ee_reciprocal}
    \begin{cases}
        \vv{\xi}_{DP} \cdot \vv{\xi}_{EE} = \vv{\xi}_{DP} \cdot \vv{\xi}_A \\
        \vv{\xi}_{GE} \cdot \vv{\xi}_{EE} = \vv{\xi}_{GE} \cdot \vv{\xi}_B \\
        \vv{\xi}_{HF} \cdot \vv{\xi}_{EE} = \vv{\xi}_{HF} \cdot \vv{\xi}_C
    \end{cases}
\end{equation}
which can be expanded as: 
\begin{equation} \label{eq:twist_reciprocal_expand}
    \begin{cases}
        (\vv{p}-\vv{d})^T\dvv{p} = \dot{\theta}_1 [\vv{x}_f \times (\vv{d}-\vv{a})]^T (\vv{p}-\vv{d}) \\
        (\vv{e}-\vv{g})^T\dvv{e} = \dot{\theta}_2 [\vv{x}_f \times (\vv{g}-\vv{b})]^T (\vv{e}-\vv{g}) \\
        (\vv{f}-\vv{h})^T\dvv{f} = \dot{\theta}_3 [\vv{x}_f \times (\vv{h}-\vv{c})]^T (\vv{f}-\vv{h})
    \end{cases}
\end{equation}

The angular velocity of the end-effector is noted $\vv{\omega}$, then we have $\dvv{e} = \dot{\theta}_1[\vv{x}_f \times (\vv{d}-\vv{a})] + \vv{\omega} \times \vv{Q}\vv{e}_b$ and $\dvv{f} = \dot{\theta}_1[\vv{x}_f \times (\vv{d}-\vv{a})] + \vv{\omega} \times \vv{Q}\vv{f}_b$. Equation \eqref{eq:twist_reciprocal_expand} then becomes: 
\begin{equation} \label{eq:dp_omega_dtheta}
    \begin{cases}
        (\vv{p}-\vv{d})^T\dvv{p} = \dot{\theta}_1 \vv{u}_1^T (\vv{d}-\vv{p}) \\
        [\vv{Q}\vv{e}_b \times (\vv{e} - \vv{g})]^T \vv{\omega} = -\dot{\theta_1} \vv{u}_1^T (\vv{e}-\vv{g}) + \dot{\theta}_2 \vv{u}_2^T (\vv{e}-\vv{g}) \\
        [\vv{Q}\vv{f}_b \times (\vv{f} - \vv{h})]^T \vv{\omega} = -\dot{\theta_1} \vv{u}_1^T (\vv{f}-\vv{h}) + \dot{\theta}_3 \vv{u}_3^T (\vv{f}-\vv{h}) \\
    \end{cases}
\end{equation}
with $\vv{u}_1=\vv{x}_f \times (\vv{d}-\vv{a})$, $\vv{u}_2=\vv{x}_f \times (\vv{g}-\vv{b})$ and $\vv{u}_3=\vv{x}_f \times (\vv{h}-\vv{c})$.

It is possible to express $\vv{\omega}$ as a function of $\dvv{p}$ through the joints velocities $\dvv{\theta}_s = [ \dot{\theta}_1 \; \dot{\theta}_{s2} \; \dot{\theta}_{s3} ]^T$: 
\begin{equation}
    \begin{cases}
        \vv{J}_\omega \dvv{\theta}_s=\vv{\omega} \\
        \vv{J}_p \dvv{\theta}_s=\dvv{p}
    \end{cases}
\end{equation}
with
\begin{equation}
    \vv{J}_\omega = 
    \begin{bmatrix}
        \vv{x}_f & \vv{x}_f & \vv{Q}\vv{z}_f
    \end{bmatrix}
\end{equation}
and
\begin{equation}
    \vv{J}_p = 
    \begin{bmatrix}
        \vv{x}_f\times(\vv{p}-\vv{a}) & \vv{x}_f\times(\vv{p}-\vv{d}) & \vv{Q}\vv{z}_f\times(\vv{p}-\vv{d})
    \end{bmatrix}
\end{equation}
Matrix $\vv{J}_p$ is singular only if $\vv{x}_f$, $\vv{p}-\vv{a}$ and $\vv{p}-\vv{d}$ are coplanar, which corresponds to a type I singularity. Consequently, $\vv{J}_p$ is invertible in the interior of the reachable workspace and: 
\begin{equation} \label{eq:omega_dp}
    \vv{\omega} = \vv{J}_\omega \vv{J}_p^{-1}\dvv{p}
\end{equation}

Then, combining \eqref{eq:dp_omega_dtheta} and \eqref{eq:omega_dp} yields: 
\begin{equation}
    \vv{J}\dvv{p} = \vv{K}\dvv{\theta}
\end{equation}
with
\begin{equation}
    \vv{J} = 
    \begin{bmatrix}
        (\vv{p}-\vv{d})^T \\
        [\vv{Q}\vv{e}_p \times (\vv{e}-\vv{g})]^T \vv{J}_\omega \vv{J}_p^{-1} \\
        [\vv{Q}\vv{f}_p \times (\vv{f}-\vv{h})]^T \vv{J}_\omega \vv{J}_p^{-1}
    \end{bmatrix}
\end{equation}
and
\begin{equation}
    \vv{K} = 
    \begin{bmatrix}
        (\vv{p}-\vv{d})^T\vv{u}_1 & 0 & 0 \\
        -(\vv{e}-\vv{g})^T\vv{u}_1 & (\vv{e}-\vv{g})^T\vv{u}_2 & 0 \\
        -(\vv{f}-\vv{h})^T\vv{u}_1 & 0 & (\vv{f}-\vv{h})^T\vv{u}_3
    \end{bmatrix}
\end{equation}

Note that an equivalent formulation can be obtained by taking the time derivative of the inverse kinematics. 

\section{Singularity Analysis} \label{sec:singularity_analysis}

The orientation capabilities of parallel robots are limited by type II (or parallel) singularities than occur within the reachable workspace. Therefore, one can expect the kinematically redundant robots to have a higher orientational workspace by avoiding singularities. 

The singularities of this architecture can be separated in two categories: those of the leg mechanism and those of the platform. This section discusses type I and type II singularities of the leg mechanism. The singularities of the platform have already been studied by Wen et al. \cite{Wen2021}. 

\subsection{Type I Singularities}

Type I singularities occur when a nonzero input velocity $\dvv{\theta}$ produces a zero output velocity $\dvv{p}$, or, equivalently, if $\det(\vv{K})=0$ \cite{Gosselin1990}. These singularities can be identified immediately from the choice of the reciprocal twists in \eqref{eq:twist_ee_reciprocal}. Indeed, the leg is in a type I singularity if (i) $\vv{\xi}_{PD}$ is reciprocal to $\vv{\xi}_{A}$ or (ii) $\vv{\xi}_{GE}$ is reciprocal to $\vv{\xi}_{B}$ or (iii) $\vv{\xi}_{HF}$ is reciprocal to $\vv{\xi}_{C}$. In other words, type I singular configurations are the ones in which the line of these reciprocal twists passes through the axis of the actuated revolute joints. It is straightforward to verify that this condition is equivalent to having a zero value of one of the diagonal entries of matrix $\vv{K}$. 

\subsection{Type II Singularities}

Type II singularities are configurations in which the Jacobian matrix $\vv{J}$ is non-invertible \cite{Gosselin1990}. They correspond to situations in which the end-effector can locally move even with the actuators locked. In this case, when the actuators are locked, the end-effector body $DEFP$ can only perform rotations around point $D$. The U joint at $D$ prevents rotations around the $\vv{x}_f \times \vv{Q}\vv{z}_f$ axis. Links $GE$ and $HF$ prevent rotations respectively around the $\vv{Q}\vv{e}_b \times (\vv{e}-\vv{g})$ axis and the $\vv{Q}\vv{f}_b \times (\vv{f}-\vv{h})$ axis. If these three axes have a linear dependency, then the leg is in a type II singularity. Hence, the following proposition. 

\begin{proposition} \label{prop:jac_singularity}
    If the end-effector of the leg $P$ is in the interior of the reachable workspace (i.e. no type I singularity), then
    \begin{equation}
        \det(\tilde{\vv{J}}) = 0 \Leftrightarrow \det(\vv{J}) = 0
    \end{equation}
    with
    \begin{equation}
        \tilde{\vv{J}} = \begin{bmatrix}
            (\vv{x}_f \times \vv{Q}\vv{z}_f)^T \\
            [\vv{Q}\vv{e}_b \times (\vv{e}-\vv{g})]^T \\
            [\vv{Q}\vv{f}_b \times (\vv{f}-\vv{h})]^T
        \end{bmatrix}
    \end{equation}
\end{proposition}

\begin{proof}
    \underline{Sufficient condition:} This case is already proven by Zhou and Gosselin \cite{Zhou2022}. Suppose that $\det(\tilde{\vv{J}}) = 0$. If $[\vv{Q}\vv{e}_b \times (\vv{e}-\vv{g})]^T$ and $[\vv{Q}\vv{f}_b \times (\vv{f}-\vv{h})]^T$ are collinear, then $\det(\vv{J}) = 0$. Otherwise, there exists a linear dependency between the rows of matrix $\tilde{\vv{J}}$. More precisely, there exist two scalars $\lambda$ and $\mu$, non both zero, such that either $[\vv{Q}\vv{e}_b \times (\vv{e}-\vv{g})]^T = \lambda [\vv{Q}\vv{f}_b \times (\vv{f}-\vv{h})]^T + \mu (\vv{x}_f \times \vv{Q}\vv{z}_f)^T$ or $[\vv{Q}\vv{f}_b \times (\vv{f}-\vv{h})]^T = \lambda [\vv{Q}\vv{e}_b \times (\vv{e}-\vv{g})]^T + \mu (\vv{x}_f \times \vv{Q}\vv{z}_f)^T$. Since $[\vv{x}_f \times \vv{Q}\vv{z}_f]^T \vv{J}_\omega = \vv{0}$, the last two rows of matrix $\vv{J}$ are collinear. Hence, $\det(\vv{J}) = 0$. \\
    \underline{Necessary condition:} Suppose that $\det(\vv{J}) = 0$. Since $\vv{J}_p$ is invertible, 
    \begin{equation}
        \begin{vmatrix}
            (\vv{p}-\vv{d})^T \vv{J}_p \\
            [\vv{Q}\vv{e}_b \times (\vv{e}-\vv{g})]^T \vv{J}_\omega \\
            [\vv{Q}\vv{f}_b \times (\vv{f}-\vv{h})]^T \vv{J}_\omega
        \end{vmatrix} = 0
    \end{equation}
    Since $(\vv{p}-\vv{d})^T \vv{J}_p = (\vv{p}-\vv{d})^T \vv{u}_1 [ \; 1 \; 0 \; 0 \; ]$ and $(\vv{p}-\vv{d})^T \vv{u}_1 \neq 0$ (no type I singularity), then
    \begin{equation}
        \begin{vmatrix}
            \vv{x}_f^T [\vv{Q}\vv{e}_b \times (\vv{e}-\vv{g})] & (\vv{Q}\vv{z}_f)^T [\vv{Q}\vv{e}_b \times (\vv{e}-\vv{g})] \\
            \vv{x}_f^T [\vv{Q}\vv{f}_b \times (\vv{f}-\vv{h})] & (\vv{Q}\vv{z}_f)^T [\vv{Q}\vv{f}_b \times (\vv{f}-\vv{h})]
        \end{vmatrix} = 0
    \end{equation}
    As a consequence, there exist two scalars $\lambda$ and $\mu$, non both zero, such that 
    \begin{equation}
        \lambda 
        \begin{bmatrix}
            \vv{x}_f^T [\vv{Q}\vv{e}_b \times (\vv{e}-\vv{g})] \\
            \vv{z}_f^T\vv{Q}^T [\vv{Q}\vv{e}_b \times (\vv{e}-\vv{g})]
        \end{bmatrix}
        = \mu
        \begin{bmatrix}
            \vv{x}_f^T [\vv{Q}\vv{f}_b \times (\vv{f}-\vv{h})] \\
            \vv{z}_f^T\vv{Q}^T [\vv{Q}\vv{f}_b \times (\vv{f}-\vv{h})]
        \end{bmatrix}
    \end{equation}
    Vectors $\vv{x}_f$, $\vv{Q}\vv{z}_f$ and $\vv{x}_f \times \vv{Q}\vv{z}_f$ form a basis of $\mathbb{R}^3$, therefore $\lambda [\vv{Q}\vv{e}_b \times (\vv{e}-\vv{g})] - \mu [\vv{Q}\vv{f}_b \times (\vv{f}-\vv{h})]$ and $\vv{x}_f \times \vv{Q}\vv{z}_f$ are collinear. Hence, $\det(\tilde{\vv{J}}) = 0$. 
\end{proof}

With no assumptions, interpreting geometrically all the type II singularities of the leg mechanism may require a large number of special cases to discuss. Therefore, a simplifying assumption is introduced that allows for an exhaustive analysis of the singularities. 

Suppose that the spherical joint centres $E$ and $F$ are coincident (i.e. $e_y=f_y=0$). 
From Proposition \ref{prop:jac_singularity}, the leg mechanism is not in a type II singularity if and only if
\begin{equation}
    \dim ( \text{span}\{ \vv{x}_f \times \vv{Q}\vv{z}_f, \vv{Q}\vv{e}_b \times (\vv{e}-\vv{g}), \vv{Q}\vv{f}_b \times (\vv{f}-\vv{h}) \} ) = 3
\end{equation}

Since $\vv{Q}\vv{e}_b \perp \text{span}(\{ \vv{Q}\vv{e}_b \times (\vv{e}-\vv{g}), \vv{Q}\vv{f}_b \times (\vv{f}-\vv{h}) \})$, the condition is verified if and only if $\vv{Q}\vv{e}_b\times(\vv{e}-\vv{g})$ and $\vv{Q}\vv{e}_b\times(\vv{e}-\vv{h})$ are linearly independent and
\begin{equation}
    (\vv{x}_f \times \vv{Q}\vv{z}_f)^T\vv{Q}\vv{e}_b \neq 0 \Leftrightarrow {\vv{x}_f}^T\vv{Q}\vv{y}_f \neq 0 \Leftrightarrow \theta_{s3} \neq \frac{\pi}{2} + k\pi
\end{equation}
This second condition only happens if $\vv{p}-\vv{d}$ and $\vv{x}_f$ are collinear and so corresponds to a type I singularity. Therefore, within the interior of its reachable workspace, the leg mechanism is in a type II singularity if and only if $\vv{Q}\vv{e}_b\times(\vv{e}-\vv{g})$ and $\vv{Q}\vv{e}_b\times(\vv{e}-\vv{h})$ are collinear, or, equivalently $\vv{e}-\vv{g}$, $\vv{e}-\vv{h}$ and $\vv{Q}\vv{e}_b$ are coplanar. 

By using remote centre of motion mechanisms (e.g., \cite{Zong2008}), it is possible to have the centres of joints $E$ and $F$ coincident. However, this solution is not considered in this work because of its mechanical complexity. The reader may refer to the Tetrobot to see a practical example in which multiple spherical joint centres are coincident \cite{Hamlin1997}. 

\section{Optimal Design} \label{sec:optimal_design}

\subsection{Considerations}

In order to use the robot for intuitive physical human robot interactions, it is prescribed that the end-effector of the legs $P$ should be able to undergo an acceleration of magnitude $2g$ in any direction, where $g$ stands for the gravitational acceleration. Therefore, it is desired to find the values of the geometric parameters that maximize the workspace in which the end-effector is able to undergo the desired accelerations with a maximum actuator torque of $\SI{10}{Nm}$. 

As discussed in Section \ref{sec:robot_architecture}, the workspace of the leg has a circular symmetry around the axis of the actuated joints. 
Therefore, it is possible to consider a half-plane that contains the axis of the actuated joints as the workspace. 
For a given value of $\norm{\vv{p}_b}$, the reachable workspace of the leg is maximized if $l_1=\norm{\vv{p}_b}$. In this case, the workspace is represented in the half-plane as a disk with centre $(0,l_1)$ and radius $l_1$. The other geometric parameters can always be chosen such that the whole disk is reachable. 

The base of the $(6+3)$-DoF robot is mounted horizontally, either in a ceiling or floor configuration. As a consequence, the gravity vector is always perpendicular to the $\vv{x}_f$ vector. For calculation purposes, the leg is considered as an $m=\SI{0.5}{kg}$ point mass located at the end-effector $P$. This value corresponds approximately to the mass of the moving parts obtained with a preliminary CAD model, and, therefore, is a conservative estimate since some moving masses are closer to the actuated joints axes. For a point in the disk, to ensure that gravity can be statically compensated, the actuators must be able to generate at the end-effector any force $\vv{f}_g \in \mathbb{F}_g$ with $\mathbb{F}_g = \{mg \; [0 \; \cos \alpha \; \sin \alpha]^T, \; 0 \leq \alpha < 2\pi\}$. Consequently, producing a $2g$ acceleration at the end-effector $P$ requires a force $\vv{f} \in \mathbb{F}$, with $\mathbb{F} = \mathbb{F}_g + 2mg\;\mathbb{B}^3 = \{ \vv{u} + 2mg \; \vv{v}, \, \vv{u} \in \mathbb{F}_g, \; \vv{v} \in \mathbb{B}^3 \}$ where $\mathbb{B}^3$ is the 3D unit ball. 
This set does not correspond to an ellipsoid, and, in particular, it is not convex. 
Indeed, consider $\vv{f}_g = mg \; [0 \; \cos \alpha \; \sin \alpha]^T$. Generating a force $\vv{f}\in\mathbb{F}$ purely along the $\vv{x}_f$ axis, i.e. in the form $\vv{f} = [f \; 0 \; 0]^T$, requires an inertial force $\vv{f}_a = \vv{f} - \vv{f}_g = [f \; -mg\cos \alpha \; -mg\sin \alpha]^T$. 
Since $\| \vv{f}_a \| \leq 2mg$, then $| f | \leq \sqrt{3mg}$. Yet, the component along the $\vv{x}_f$ axis of some forces in $\mathbb{F}$, such as $\vv{f} = [2mg \; mg\cos \alpha \; mg\sin \alpha]^T$, are larger than $\sqrt{3mg}$. Hence, $\mathbb{F}$ is not convex. 

The legs are in the assembly configuration illustrated in Figs. \ref{fig:robot_architecture}, \ref{fig:leg_parametrization}, \ref{fig:cad_9dof}. 
In order to have a symmetric leg, the following constraints are considered: $l_2=l_3$, $l_4=l_5$, $l_6=l_7$, $e_x=f_x$, $e_y=-f_y$ and $e_z=f_z$. 
As explained in Section \ref{sec:singularity_analysis}, the value $e_y$ is chosen as small as possible to avoid type II singularities. This also reduces the size of the leg end-effector body and thus prevents interference with other legs. For the same reason, we also choose $e_z=0$. Lastly, to minimize the footprint of the $(6+3)$-DoF robot without causing interference between legs, we restrict $l_6$ to the minimum mechanically feasible value, which is $l_6=\SI{0.130}{m}$ (obtained from the CAD model). 
Therefore, the optimization problem has three variables: $l_2$, $l_4$, $e_x$. 

The leg collides with the fixed platform if $\max(\vv{\theta}) > 270\degree$. The optimization problem is only solved in the horizontal half-plane $z = 0$ and $y \geq 0$, and we limit $\max(\vv{\theta})$ to $135\degree$.

\subsection{Optimal Solution}

The radius of the reachable workspace disk is set to be $l_1 = \SI{0.35}{m}$ to allow for a large workspace of the leg, and hence of the $(6+3)$-DoF robot. The optimization is performed by discretizing all the variables and the planar workspace. The variables are first discretized with a $\SI{1}{cm}$ step and a second discretization is performed with a $\SI{1}{mm}$ step in the neighbourhood of the best solution. The length of the links are limited to $\SI{0.35}{m}$ to enable a high stiffness of the robot without increasing significantly the mass. A solution is considered only if the RUS chains do not limit the reachable workspace of the RU chain. 
The optimal values of the variables are obtained as follows: 
\begin{equation*}
    l_2 = \SI{0.139}{m}, l_4 = \SI{0.350}{m}, e_x = \SI{0.097}{m}
\end{equation*}

Figure \ref{fig:det_JK} shows the values of the determinants of matrices $\vv{J}$ and $\vv{K}$. It can be seen that there are type II singularities in the interior of the reachable workspace, but they are close to the boundaries. 

\begin{figure}[htbp]
    \centering
    \subfloat[Determinant of $\vv{J}$.]{%
        \includegraphics[width=0.45\linewidth]{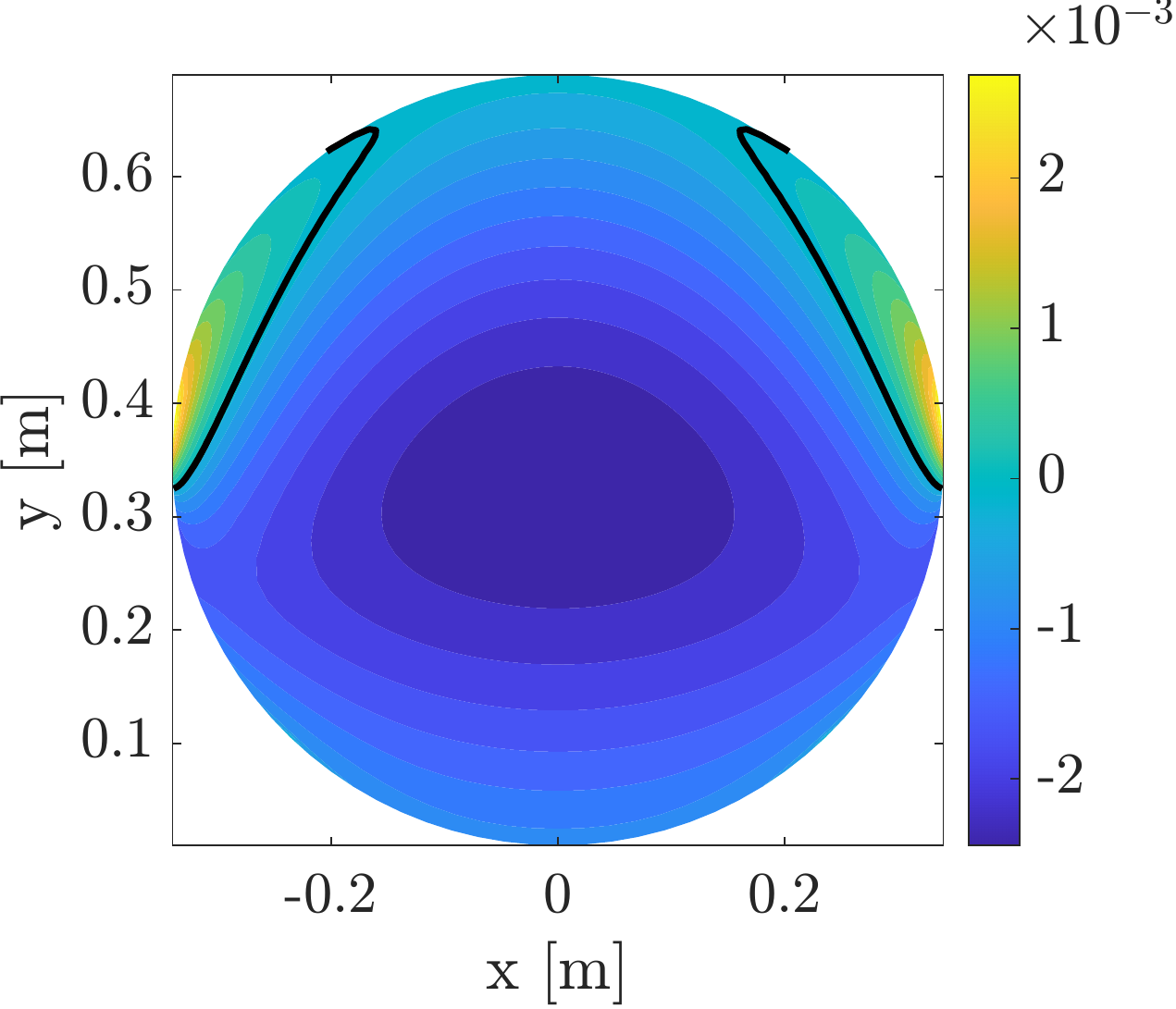}}
    \hfill
    \subfloat[Determinant of $\vv{K}$.]{%
        \includegraphics[width=0.45\linewidth]{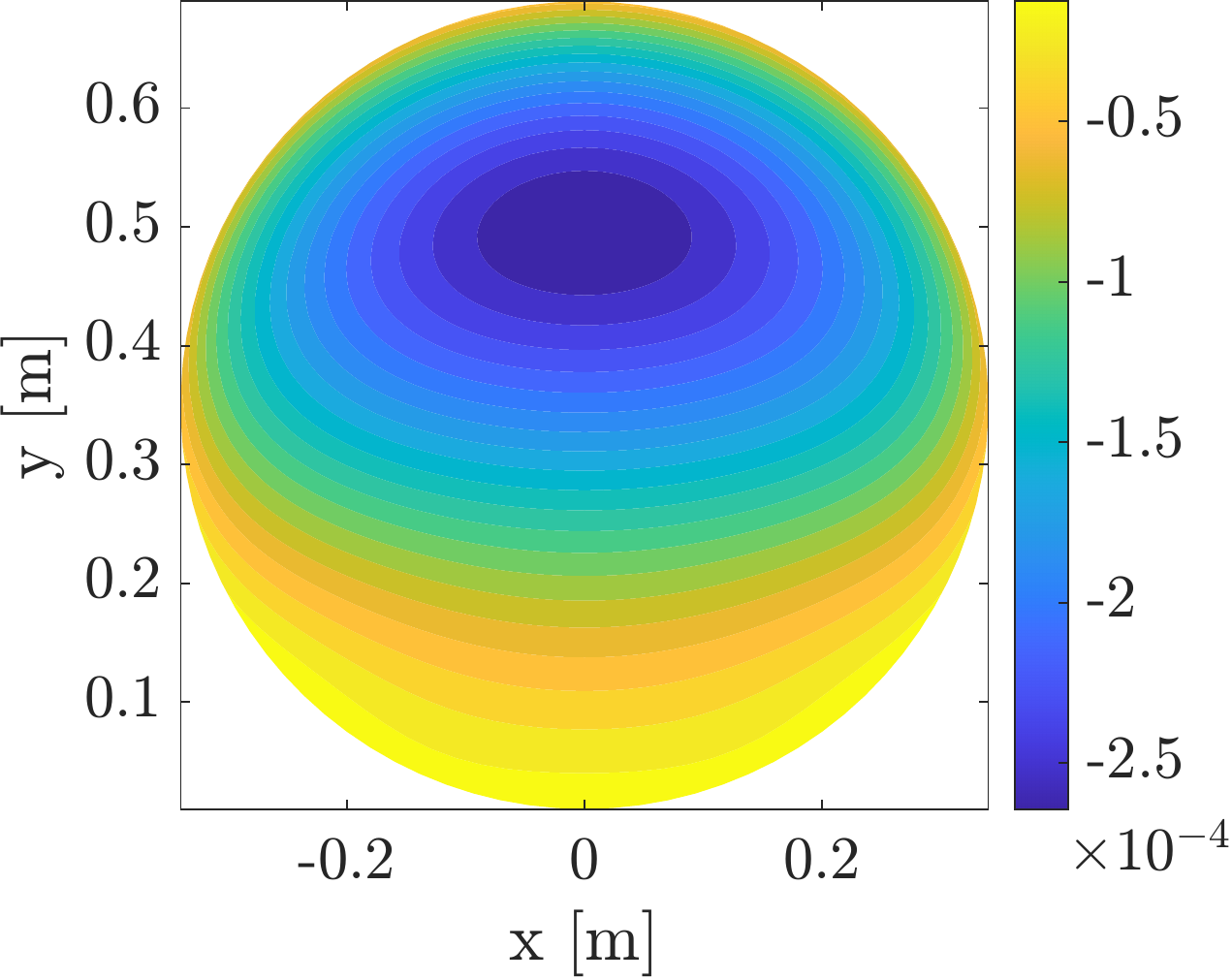}}
    \caption{Determinant of Jacobian matrices. The black lines show type II singularities. }
    \label{fig:det_JK} 
\end{figure}

Figure \ref{fig:WFW_optimal} shows the maximum actuator torque $\tau_{ext}$ that is necessary to generate the required forces $\mathbb{F}_g + 2mg\;\mathbb{B}^3$ with the optimal solution. For better readability, $\tau_{ext}$ is saturated at $\SI{11}{Nm}$. 

\begin{figure}[htbp]
    \centering
    \includegraphics[width=0.30\textwidth]{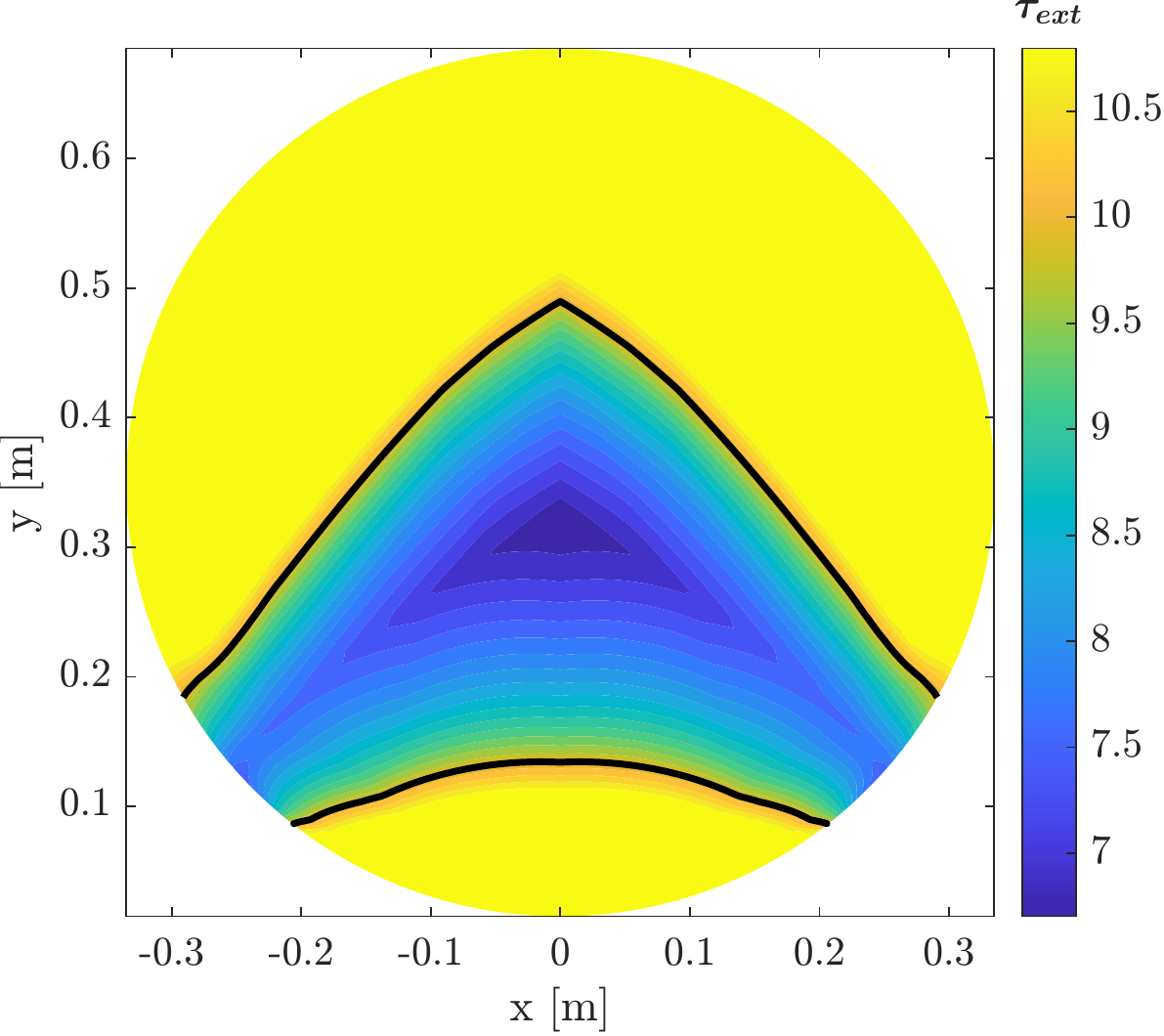}
    \caption{Optimal workspace. Black lines correspond to $\tau_{ext}=\SI{10}{Nm}$. }
    \label{fig:WFW_optimal}
\end{figure}

The singularity-free COW of the (6+3)-DoF robot is illustrated in Fig. \ref{fig:cow_9dof} with the following considerations. The legs are positioned symmetrically (as in Fig. \ref{fig:robot_architecture}), with $\| \vv{a}_j-\vv{a}_i \|=\SI{0.143}{m}$ ($i$ and $j$ refer to the leg number). The end-effector reconfigurable platform is considered a horizontal equilateral triangle of side length $\SI{0.173}{m}$. With these considerations, the ratio between the area of the COW projected on a horizontal plane and the footprint (defined as the area of the convex hull of actuated joint positions) is $5.03$. 

\begin{figure}[htbp]
    \centering
    \includegraphics[trim=0px 105px 0px 105px,clip=true,width=0.30\textwidth]{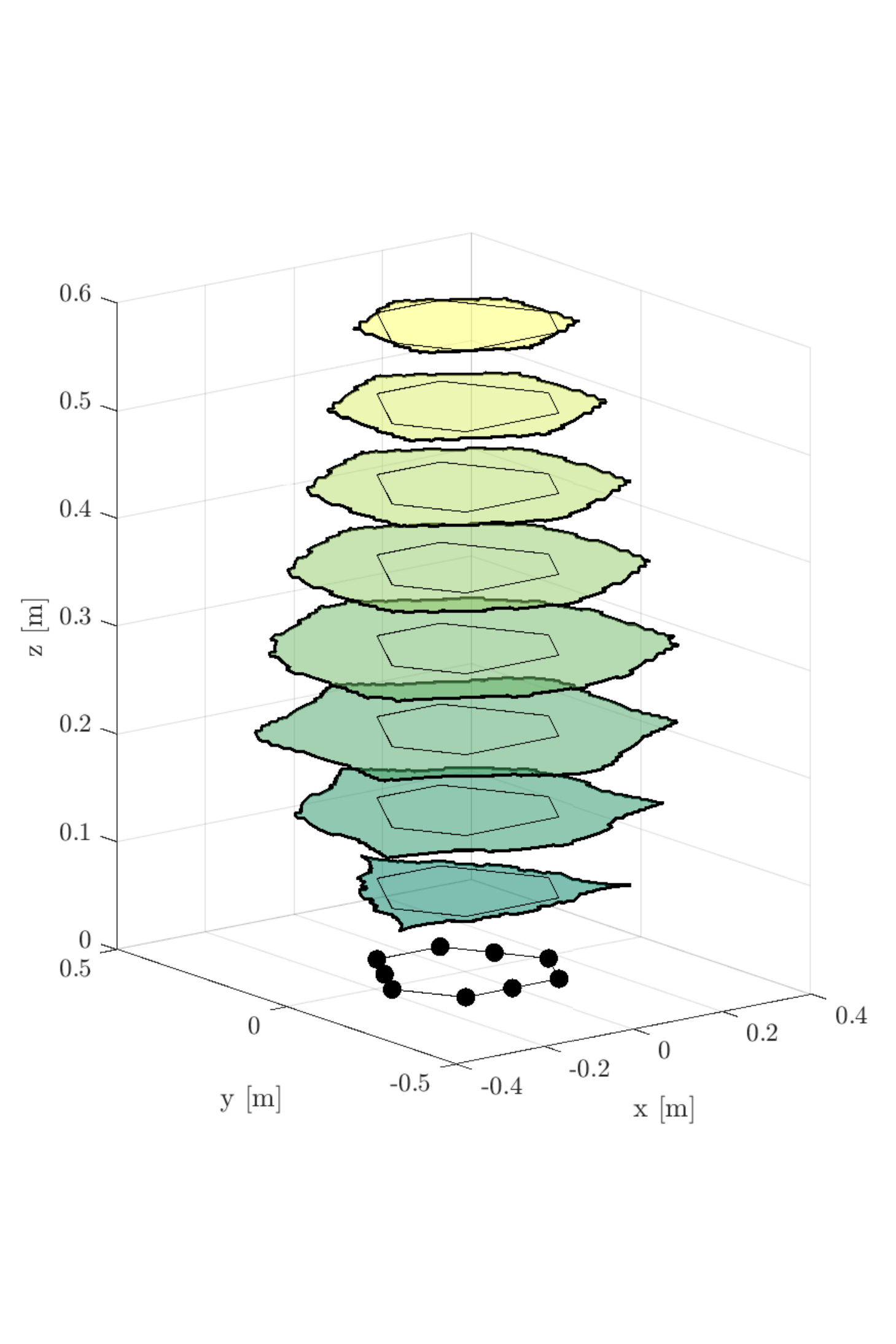}
    \caption{Constant orientation workspace of the $(6+3)$-DoF robot. The polygon that is drawn on each layer is the convex hull of the positions of actuated joints (black dots). }
    \label{fig:cow_9dof}
\end{figure}

\section{Prototyping} \label{sec:prototyping}

A prototype of the proposed $(6+3)$-DoF robot is under construction. Figure \ref{fig:cad_9dof} shows a 3D model obtained with a CAD software. 

\begin{figure}[htbp]
    \centering
    \includegraphics[width=0.35\textwidth]{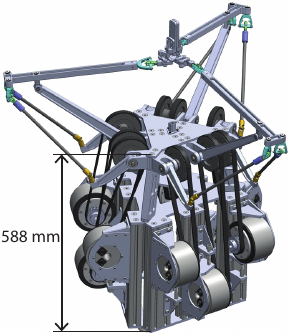}
    \caption{CAD model of the (6+3) robot.}
    \label{fig:cad_9dof}
\end{figure}

The prototype uses nine identical SGMCS-05B3B-YAX1 direct-drive motors with $\SI{5}{Nm}$ nominal torque and nine SGD7S-2R8FA0A drives from Yaskawa. Timing belt transmissions allow for reducing the footprint of the robot and multiplying the available torque by $72/34\approx2.12$. 
The actuators are equipped with $20$-bit encoders and, therefore, the resolution of the end-effector of the leg $P$ is $\SI{8.4e-6}{m}$ in the worst case scenario. 

The limited mechanical stress on the distal links of the RUS chains allows for using lightweight carbon fibre tubes. The remaining parts are made of aluminum. 

The range of motion of the spherical joints is critical to benefit from the large workspace of the leg. Since the motion of traditional spherical joints is limited by mechanical interference, the 4-DoF spherical joints introduced in \cite{Schreiber2017} are used and adapted to meet the mechanical stress requirements of the prototype. These joints can generate a very large range of motion that exceeds $\pm 150 \degree$. The video material accompanying this paper shows an animation of the model constructed with the CAD software performing rotations and translations with large ranges of motion (\url{https://www.youtube.com/watch?v=lBKiMFoyJ8o}).

\section{Conclusion and Future Work} \label{sec:conclusion}

This paper introduces a backdrivable and kinematically redundant $(6+3)$-DoF parallel robot for sensorless physical human-robot interaction. The robot has three identical legs and all actuators are fixed to the base. The leg mechanism is designed to maximize the workspace in which the end-effector of the leg can have a $2g$ acceleration in all directions. The preliminary design of the robot is shown using a CAD model. The design is selected to minimize the footprint of the $(6+3)$-DoF robot while allowing very large ranges of rotation and translation. 

The prototype is under construction. Current work includes the dynamic modelling of the robot, the development of control algorithms for physical human-robot interaction and the design of configurable platforms actuated by the redundant degrees of freedom.

\bibliographystyle{IEEEtran}
\bibliography{references}

\end{document}